\pdfoutput=1

\documentclass{article}
\usepackage{ijcai16}
\usepackage{graphicx,times,multirow} 
\usepackage{algorithm}
\usepackage{clrscode}
\usepackage{algpseudocode}
\usepackage{enumerate}
\usepackage{caption}
\usepackage{colortbl}
\usepackage{longtable}
\usepackage{array}
\usepackage{amsmath}
\usepackage{amssymb}
\usepackage{amsthm}
\usepackage{pst-blur}
\usepackage{bm}
\usepackage{pstricks-add}
\usepackage{changepage}
\usepackage{subfigure}
\usepackage{epstopdf}
\usepackage{fancyhdr} 
\definecolor{Blue}{rgb}{1.0,0.75,0.8}
\usepackage{times}

\graphicspath{{images/}}

\newtheorem{Lemma}{Lemma}
\newtheorem{Definition}{Definition}

\title{High-dimensional Black-box Optimization \\ via Divide and Approximate Conquer}

\author{Peng Yang$^1$, Ke Tang$^1$\footnote{Corresponding author}, and Xin Yao$^2$\\
$^1$UBRI, School of Computer Science and Technology,\\ University of Science and Technology of China, Hefei, China, 230027\\
$^2$CERCIA, School of Computer Science,\\ University of Birmingham, Birmingham B15 2TT, U.K.\\
Emails: trevor@mail.ustc.edu.cn; ketang@ustc.edu.cn; x.yao@cs.bham.ac.uk}

\begin{document}

\maketitle

\begin{abstract}
Divide and Conquer (DC) is conceptually well suited to high-dimensional optimization by decomposing a problem into multiple small-scale sub-problems. 
However, appealing performance can be seldom observed when the sub-problems are interdependent. 
This paper suggests that the major difficulty of tackling interdependent sub-problems lies in the precise evaluation of a partial solution (to a sub-problem), which can be overwhelmingly costly and thus makes sub-problems non-trivial to conquer. 
Thus, we propose an approximation approach, named Divide and Approximate Conquer (DAC), which reduces the cost of partial solution evaluation from exponential time to polynomial time. 
Meanwhile, the convergence to the global optimum (of the original problem) is still guaranteed. 
The effectiveness of DAC is demonstrated empirically on two sets of non-separable high-dimensional problems.
  
\end{abstract}

\section{Introduction}

Developing Artificial Intelligence (AI) applications often encounters a key task of solving challenging optimization problems. 
Formally, it can be stated as: 

\quad\quad\quad\quad\quad\quad$\textbf{x}^* = {\arg\max}_{\textbf{x}\in \mathcal{X}}f(\textbf{x})$,
 
\noindent where $f:\mathcal{X}\rightarrow \mathbb{R}$ denotes a function on a bounded solution space $\mathcal{X}\subseteq \mathbb{R}^D$ and $\textbf{x}^*$ denotes the global optimum in $\mathcal{X}$. 
We consider $f$ as a black-box function that the problem information is completely unknown beforehand, where only the function value of $\textbf{x}\in \mathcal{X}$ can be directly provided if explicitly queried. 
Therefore, only derivative-free approaches can be brought to bear, which, however, often suffer from the curse of dimensionality if $D$ is considerably large. 

An intuitive idea to handle a high-dimensional optimization problem is to project its solution space onto lower dimensions, where traditional approaches perform well \cite{kaban2015toward}.
However, it is nontrivial to identify an appropriate projection.
Typical approaches in this category, e.g., Random Embedding techniques \cite{wang2013bayesian,qian2016scaling}, consider the high-dimensional problem having low effective dimensionality, for which a random projection would suffice to find the global optimal solution of a high-dimensional problem in a low-dimensional space.
Although these algorithms also showed appealing performance in case the low effective dimensionality assumption is mildly violated, their performance may not be satisfactory on the wider range of irreducible problems. 

Divide-and-Conquer (DC) is another general idea for tackling large-scale problems.
In case of high-dimensional black-box optimization, DC can be implemented by dividing the original problem into multiple sub-problems (say of dimensionality $d_i$, where $i=1,...,M$, and $d_i<D$) \cite{yang2008large}. 
For the $i$-th sub-problem, $\textbf{x}$ is optimized along $d_i$ dimensions, while its values on the other $D-d_i$ dimensions are fixed.
That is, each sub-problem concerns a low-dimensional subspace of the original solution space.
Applying an existing search method to the $i$-th sub-problem leads to a $d_i$-dimensional partial solution to the original high-dimensional problem.
The solution to original problem can be obtained by combining the partial solutions achieved on all sub-problems.

Given an appropriate sub-problems optimizer, the above-described DC strategy works well on the so-called separable problems, for which the global optimal optimum can be found by optimizing one dimension at a time regardless of the values taken on the other dimensions \cite{chen2010large}.
If this condition does not hold, the performance of DC heavily relies on the decomposition method \cite{omidvar2014cooperative}, which aims to divide a black-box high-dimensional problem in such a way that the global optimum can still be obtained by solving the sub-problems in a fully independent manner.
In the past few years, a large variety of decomposition methods have been proposed \cite{mahdavi2015metaheuristics}.
Despite the performance enhancement brought by them, none of these methods are guaranteed to achieve the desired sub-problems.
Meanwhile, a practical problem of interest may be fully non-separable such that the ideal decomposition mentioned above does not even exist.
Therefore, how to deal with (conquer) the interdependent sub-problems remains a challenge to DC in the context of high-dimensional optimization.

In this paper, we suggest that the major challenge of tackling interdependent sub-problems lies in the difficulty of evaluating the quality of a partial solution (to a sub-problem) during the search course.
To be specific, as the quality of a partial solution (to the original problem) depends on the values taken on the dimensions involved in other sub-problems, precisely evaluating a partial solution requires overwhelming computation cost, which increases exponentially with the number of interacting variables in other sub-problems. 
We propose an approximation approach for partial solution evaluation, which yields a novel framework, named Divide and Approximate Conquer (DAC).
With DAC, the computational cost increases polynomially with the number of solutions generated in each iteration while the convergence to global optimum (of the original problem) is still guaranteed.

The major difficulty of tackling interdependent sub-problems is analyzed in Section 2.
The proposed DAC is detailed in Section 3. 
Section 4 reports empirical studies of DAC on five synthetic high-dimensional optimization problems and the hyper-parameters fine-tuning task for multi-class SVM.
Section 5 concludes this work.

\section{Major challenge of dealing with interdependent sub-problems}
The DC strategy consists of three steps:
\begin{itemize}
\item[1)] (\textbf{Divide}) Decompose a problem into $M$ $d_i$-dimensional sub-problems, where $i=1,...,M$ and $d_i<D$;
\item[2)] (\textbf{Conquer}) Search the best partial solution in each sub-problem by applying an existing search approach;
\item[3)] Merge the best partial solutions obtained on all sub-problems as the final output. 
\end{itemize}

We restrict our discussions here to the Conquer phase.
Usually, a derivative-free search process in each sub-problem is guided by the solutions with better function values, despite a few exceptions \cite{Kirkpatrick671,tang2016negatively}. 
Before a $d_i$-dimensional partial solution is evaluated, it must be complemented to $D$-dimensional by fixing the values for the variables in other sub-problems. 
Specially, we call a vector of such fixed $D-d_i$ values as a \textit{complement} to a partial solution.
A partial solution will receive different function values with different complements.
Fortunately, the indeterminate function values of partial solutions will not influence the search process unless the rank of partial solutions changes, on which the search direction is actually determined. 
The rank of partial solutions may change if sub-problems are interdependent, which is defined as follows:

\begin{Definition}{\rm{(Interacting Variables) \cite{chen2010large}}}  \\
\label{de1}
Given a $D$-dimensional problem, its $i$-th and $j$-th variables are said to be interacting, if the rank of two partial solutions $x_i$ and $x_i'$ in the $i$-th dimension may change with different complements, e.g., $x_j$ and $x_j'$, in the $j$-th variable:
\begin{equation*}
\label{eq1}
\begin{split}
&\exists \textup{\textbf{x}},x_i',x_j':\\
&f(x_1,...,x_i,...,x_j,...,x_D)<f(x_1,...,x_i',...,x_j,...,x_D) \land \\
&f(x_1,...,x_i,...,x_j',...,x_D)>f(x_1,...,x_i',...,x_j',...,x_D). 
\end{split}
\end{equation*} 
\end{Definition}

\begin{Definition}{\rm{(Interdependent Sub-problems)}}\\
\label{de2}
Given two arbitrary sub-problems, they are said to be interdependent if at least one variable in a sub-problem is interacting with at least one variable in the other sub-problem. 
\end{Definition}

\noindent Intuitively, two interacting variables of the 2-D Schwefel function are depicted in Figure 1, where the rank of two partial solutions $x_1$ and $x_1'$ in the first dimension varies by fixing different values in the second dimension, i.e., $x_2$ and $x_2'$.

\begin{figure}\renewcommand{\captionfont}{\footnotesize}
			 \centering
			\begin{minipage}[htb]{1\linewidth}
				\centering
				\includegraphics[width=0.8 \linewidth, height=0.6 \linewidth]{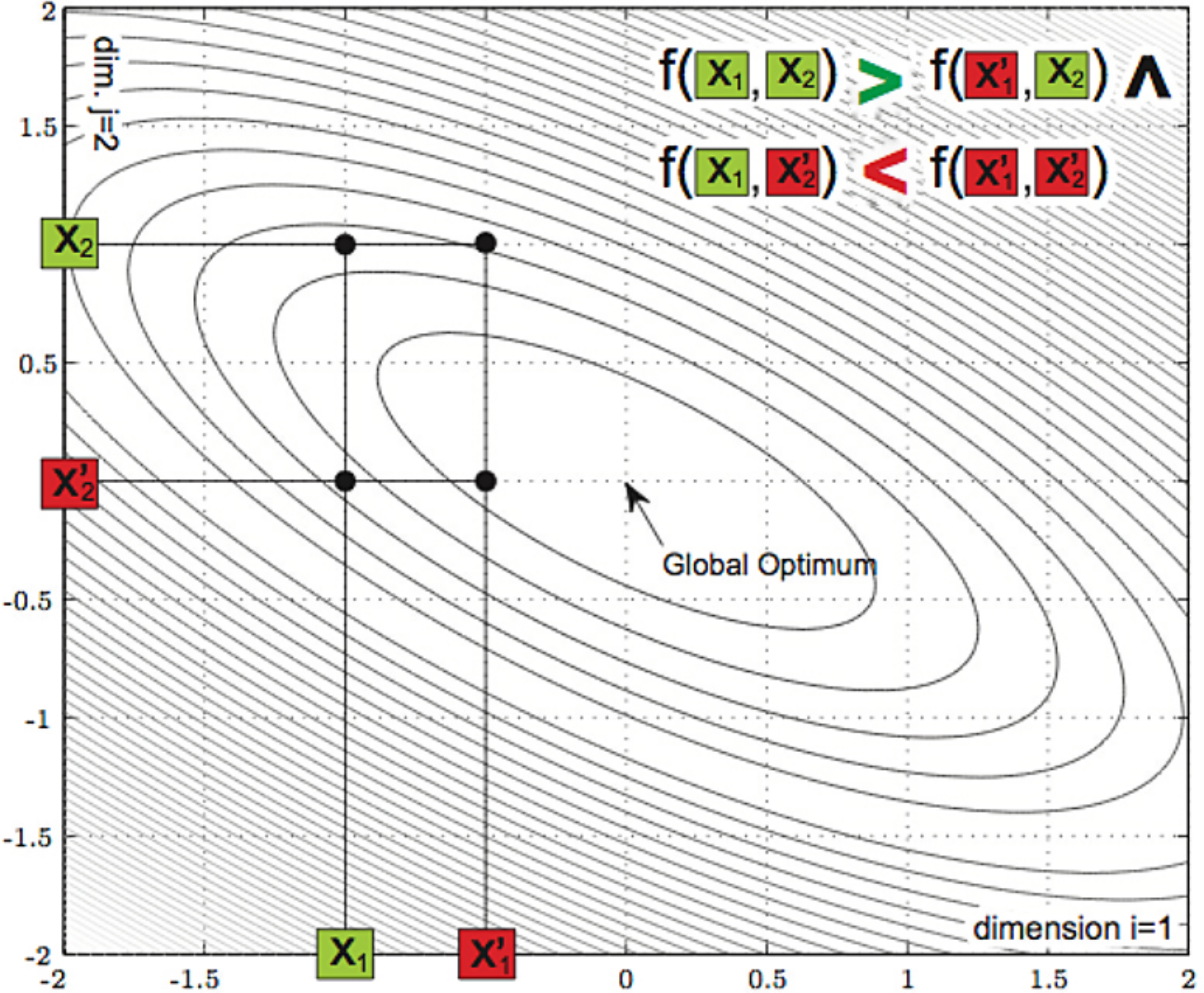}
				\caption{Two interacting variables of the 2-D Schwefel function.}
			\end{minipage}%
\end{figure}

If a problem is separable, where no interdependency exists between sub-problems, partial solutions can be complemented by arbitrary identical values in $\mathcal{X}$, without perturbing their rank. 
However, for many non-separable problems, the sub-problems are interdependent, where the rank of partial solutions significantly relies on their complements.
As a result, the search direction has a close relation to the choice of complements.
Hence, the choice of complements to partial solutions should be carefully addressed.
Otherwise, the search process in a sub-problem will run the risk of being misled, eventually resulting in an ineffective search. 

Unfortunately, we find that how to accurately complement partial solutions is a difficult optimization problem.

\begin{Lemma}{\rm{(The Difficulty of Accurately Complementing)}}\\
\label{le1}
Given a set of partial solutions in the $d_i$-dimensional $i$-th sub-problem, let $\widehat{D_i}$ be the set of all $D$ variables except the ones in the $i$-th sub-problem, $\lvert\widehat{D_i}\rvert$ be the cardinality of $\widehat{D_i}$, $P_{\widehat{D_i}(j)}$ be the probability of fixing a correct value for the $j$-th variable in $\widehat{D_i}$, and $P$ be the arithmetic mean of all $P_{\widehat{D_i}(j)},j=1,...,D-d_i$, then the probability of accurately complementing those partial solutions so that they can be correctly ranked, denoted as $P_{i}$, is:  

\begin{equation}
P_{i}=\prod_{j=1}^{\lvert\widehat{D_i}\rvert} P_{\widehat{D_i}(j)} \leq \bigg(\frac{\sum_{j=1}^{\lvert\widehat{D_i}\rvert} P_{\widehat{D_i}(j)}}{\lvert\widehat{D_i}\rvert}\bigg)^{\lvert\widehat{D_i}\rvert} = P^{D-d_i}
\label{eq2}
\end{equation}   
\end{Lemma}

\begin{proof}
By Definition \ref{de1}, we learn that fixing the correct values for variables are independent events.
By Definition \ref{de2}, we know that variables in the same sub-problem are non-interacting.    
Thus, we can directly have 
$P_{i}=\prod_{j=1}^{\lvert\widehat{D_i}\rvert} P_{\widehat{D_i}(j)}$. 
Finally, according to the AM-GM Inequality, we have Eq.(\ref{eq2}).

\end{proof} 

\noindent Notice that, $P_{\widehat{D_i}(j)}=1$ if the $\widehat{D_i}(j)$-th variable is not interacting with any variable in the $i$-th sub-problem.
Lemma 1 shows that the computational cost of accurately complementing partial solutions increases exponentially with the number of variables that are interacting with the current sub-problem.
Hence, the interdependent sub-problems cannot be accurately conquered within a reasonable time budget. 

\section{Divide and Approximate Conquer}
\subsection{The Accurate Complement}
According to Lemma 1, only the brute-force method is applicable to accurately complement partial solutions on interdependent sub-problems. 
However, the required computational costs are beyond being acceptable.
On the other hand, it is an effective way to derive the mathematical formulation of a problem by observing its corresponding brute-force method, as such a method usually has to scan the whole solution space and thus reflects the problem characteristics naturally.

The core idea of the brute-force method is mathematically described as follow (\textit{ the maximization case is considered}):

\begin{equation}
\label{eq3}
\begin{split}
  &\forall \textbf{x}_i \in S_i: \\
  &f(\textbf{x}^*)=f(\textbf{x}_i^*, \textbf{x}_\textbf{r}^*)= \underset{\textbf{x}_\textbf{r} \in S_\textbf{r}}{\max}f(\textbf{x}_i^*,\textbf{x}_\textbf{r}) \geq \underset{\textbf{x}_\textbf{r} \in S_\textbf{r}}{\max}f(\textbf{x}_i,\textbf{x}_\textbf{r}),
\end{split}
\end{equation} 
  
\noindent where $\textbf{x}_i$ denotes a partial solution in the $i$-th sub-problem, $i=1,...,M$, and $\textbf{x}_\textbf{r}$ denotes its candidate complement, where $\textbf{r}=[1,...,i-1,i+1,...,M]$ denotes all the sub-problems except the $i$-th one. 
$S_i$ and $S_\textbf{r}$ denote the corresponding subspace of the solution space subjecting to $\lvert S_i \rvert \cdot \lvert S_\textbf{r} \rvert = \lvert S \rvert$, where $S \subseteq \mathcal{X}$ and $\textbf{x}^*=(\textbf{x}_i^*, \textbf{x}_\textbf{r}^*)$ is the optimal solution in $S$.

Eq.(\ref{eq3}) states a fact that the correct rank of partial solutions can be obtained by comparing their best function values
among all combinations with all possible complements. 
On this basis, we mathematically define the problem of accurately complementing partial solutions as follow:

\begin{Definition}{\rm{(The Accurate Complement)}}\\
\label{de3}
Given arbitrary $\textup{\textbf{x}}_i \in S_i$, a complement $\textrm{\textbf{x}}_\textup{\textbf{r}}^{\dagger} \in S_\textup{\textbf{r}}$ is said to be the accurate complement $\iff \textup{\textbf{x}}_\textup{\textbf{r}}^{\dagger} = \underset{\textup{\textbf{x}}_\textup{\textbf{r}} \in S_\textup{\textbf{r}}}{\arg\max}f(\textup{\textbf{x}}_i,\textup{\textbf{x}}_\textup{\textbf{r}}). $
\end{Definition}

Notice that, every partial solution has its own accurate complement.
To identify the accurate complement, the combinations of $\textbf{x}_i$ and all possible complements in $S_\textbf{r}$ should be evaluated by $f$, among which the complement with the largest function value is chosen. 

\subsection{DAC: an approximate approach to DC}
In fact, Definition 3 allows us to approximate the accurate complements of partial solutions with only a limited set of candidate complements $S_\textbf{r}' \subseteq S_\textbf{r}$.
That is,

\begin{equation}
\label{eq4}
\textbf{x}_\textbf{r}^{\dagger} = \underset{\textbf{x}_\textbf{r} \in S_\textbf{r}}{\arg\max}f(\textbf{x}_i,\textbf{x}_\textbf{r}) \succeq \widetilde{\textbf{x}}_\textbf{r}^{\dagger} = \underset{\textbf{x}_\textbf{r}' \in S_\textbf{r}' \subseteq S_\textbf{r}}{\arg\max}f(\textbf{x}_i,\textbf{x}_\textbf{r}').
\end{equation}

\noindent where $\textbf{x}_\textbf{r}^{\dagger} \succeq \widetilde{\textbf{x}}_\textbf{r}^{\dagger}$ means that $\textbf{x}_\textbf{r}^{\dagger}$ is more accurate than $\widetilde{\textbf{x}}_\textbf{r}^{\dagger}$, since,

\begin{equation}
\label{eq5}
\underset{\textbf{x}_\textbf{r} \in S_\textbf{r}}{\max}f(\textbf{x}_i,\textbf{x}_\textbf{r}) \geq \underset{\textbf{x}_\textbf{r}' \in S_\textbf{r}' \subseteq S_\textbf{r}}{\max}f(\textbf{x}_i,\textbf{x}_\textbf{r}').
\end{equation}  
 
\noindent Based on Eq.(\ref{eq5}), it is reasonable to assume that good \textit{approximate complements} will perturb the rank of partial solutions slightly.
Eq.(\ref{eq4}) thus gives rise to the proposed Divide and Approximate Conquer (DAC), as shown in Algorithm 1.

\begin{algorithm}[t] 
   \caption{DAC($f$, $T_{max}, N$)} 
   \label{DAC} 
   \begin{algorithmic}[1] 
       \State Randomly initialize $N$ solutions $\textbf{x}_{1:N}$.  
       \State Divide $f$ into $M$ sub-problems.
       \State \textbf{For} $t=1$ \textbf{to} $T_{max}$
           \State \quad \textbf{For} $i = 1$ \textbf{to} $M$ 
           \State \quad\quad  $\textbf{x}_{1:N;i}'= \textrm{SearchOperator}(\textbf{x}_{1:N;i})$.  
           \State \quad\quad \textbf{For} $j = 1$ \textbf{to} $N$ 
           \State \quad\quad\quad  $\widetilde{\textbf{x}}_{j;\textbf{r}}^{\dagger} = \underset{\textbf{x}_{\textbf{r}} \in \textbf{x}_{1:N;\textbf{r}}}{\arg\max}f(\textbf{x}_{j;i},\textbf{x}_{\textbf{r}})$.
           \State \quad\quad\quad  $\widetilde{\textbf{x}'}_{j;\textbf{r}}^{\dagger} = \underset{\textbf{x}_{\textbf{r}} \in \textbf{x}_{1:N;\textbf{r}}}{\arg\max}f(\textbf{x}'_{j;i},\textbf{x}_{\textbf{r}})$.
		   \State \quad\quad \textbf{EndFor}  
           \State \quad\quad  $\textbf{x}_{1:N;i} \leftarrow \{\textbf{x}_{1:N;i}; \textbf{x}'_{1:N;i}\mid \widetilde{\textbf{x}}_{1:N;\textbf{r}}^{\dagger}; \widetilde{\textbf{x}'}_{1:N;\textbf{r}}^{\dagger}\}$ .
           \State \quad\quad  $\textbf{x}_{1:N;\textbf{r}} \leftarrow \{\widetilde{\textbf{x}}_{1:N;\textbf{r}}^{\dagger}; \widetilde{\textbf{x}'}_{1:N;\textbf{r}}^{\dagger}\}$.
		   \State \quad \textbf{EndFor}
       \State \textbf{EndFor}  
       \State \textbf{Output} the best solution found so far.
   \end{algorithmic} 
\end{algorithm} 

DAC shares the same framework as the basic DC. 
The only difference between them is that: while complementing a partial solution, DAC always selects the complement associated with the largest function value among a given set of candidate complements.
Specifically, DAC works by first randomly initializing $N$ solutions $\textbf{x}_{1:N}$ (step 1).
The problem $f$ is decomposed into $M$ sub-problems with a certain decomposition strategy (step 2). 
After that, without loss of generality, let us consider the $i$-th sub-problem.
$N$ new partial solutions $\textbf{x}_{1:N;i}'$ are generated by applying some search operator to the current ones, i.e., $\textbf{x}_{1:N;i}$ (step 5).
To identify the approximate complement to the $j$-th partial solution $\textbf{x}_{j;i}$, $j=1,...,N$, all the combinations of $\textbf{x}_{j;i}$ and the vectors of partial solutions in other sub-problems $\textbf{x}_{1:N;\textbf{r}}$ will be evaluated by $f$.
The vector associated with the largest function value is chosen as the approximate complement $\widetilde{\textbf{x}}_{j;\textbf{r}}^{\dagger}$ to $\textbf{x}_{j;i}$ (step 7). 
The same strategy is used to obtain the approximate complement $\widetilde{\textbf{x}'}_{j;\textbf{r}}^{\dagger}$ to the $j$-th new partial solution $\textbf{x}_{j;i}'$ (step 8).
After that, according to a certain selection criterion, $N$ partial solutions will be remained for the next iteration (step 10).
Notice that, the selection of a partial solution is conditioned by its corresponding approximate complement. 
At last, for the $j$-th selected partial solution $\textbf{x}_{j;i}$, its corresponding partial solutions in the rest sub-problems $\textbf{x}_{j;\textbf{r}}$ will be replaced with its approximate complement for further optimizing (step 11). 

As a result, DAC consumes $2MN^2$ Function Evaluations (FEs) in each iteration, which is a significant reduction to the brute force method. 
Meanwhile, albeit the computational time is cut down, the convergence of DAC is still guaranteed.


\begin{Lemma}{\rm{(The Convergence of DAC)}}\\
\label{le2}
Given a search algorithm that can converge to the global optimum of each sub-problem (regarded as independent problems), DAC can approximately converge to the global optimum $\textup{\textbf{x}}^*$ of the original problem.
\end{Lemma}

\begin{proof}
With out loss of generality, let us consider the function values of the $j$-th solution at the $t$-th and $t+1$-th iteration, i.e., $f(\textbf{x}_j^{t})$ and $f(\textbf{x}_j^{t+1})$.
For the $i$-th sub-problem, we have:

\begin{equation}
\label{eq6}
f(\textbf{x}^t_{j;i},\textbf{x}^t_{j;\textbf{r}}) \leq \underset{\textbf{x}^t_\textbf{r} \in \textbf{x}^t_{1:N;\textbf{r}}}{\max}f(\textbf{x}^t_{j;i},\textbf{x}^t_{\textbf{r}}) \leq \underset{\textbf{x}^t_\textbf{r} \in \textbf{x}^t_{1:N;\textbf{r}}}{\max}f(\textbf{x}^{t+1}_{j;i},\textbf{x}^t_{\textbf{r}}).
\end{equation} 

\noindent where the first '$\leq$' indicates the procedure of identifying the approximate complements for the current partial solutions (step 7 in Algorithm 1), and the second '$\leq$' represents the procedures of generating new partial solutions (step 5 in Algorithm 1), identifying their approximate complements (step 8 in Algorithm 1), and selecting better ones from candidate partial solutions, conditioned by their approximate complements (step 10 in Algorithm 1).

Then by repeating Eq.(\ref{eq6}) for $i=1,...,M$, we have that:

\begin{equation}
\label{eq7}
f(\textbf{x}^t_{j;1},...,\textbf{x}^t_{j;M}) \leq f(\textbf{x}^{t+1}_{j;1},...,\textbf{x}^{t+1}_{j;M}),
\end{equation} 

\noindent which means the function value of the $j$-th solution $f(\textbf{x}_j^{t})$ monotonically increases with the iteration index $t$.

Note that, the equality of Eq.(\ref{eq7}) holds in two cases:
\begin{itemize}
\item[1)] The approximate complement of a partial solution happens to be its corresponding partial solutions in the rest sub-problems, i.e., $\textbf{x}_{j;\textbf{r}}=\widetilde{\textbf{x}}_{j;\textbf{r}}^{\dagger}$;
\item[2)] The search algorithm fails to produce new better solutions, i.e., $\underset{\textbf{x}_\textbf{r} \in S_\textbf{r}}{\max}f(\textbf{x}_{j;i},\textbf{x}_\textbf{r}) \geq \underset{\textbf{x}_\textbf{r} \in S_\textbf{r}}{\max}f(\textbf{x}'_{j;i},\textbf{x}_\textbf{r})$.
\end{itemize}

\noindent The first case actually explains the term "approximate" in DAC, as it happens at a probability of at least $\frac{1}{N}$. 
Hence, if the sub-problems optimizer of DAC can optimally solve each sub-problem separately, the global optimum value $f(\textbf{x}^*)$ can be approximately approached by DAC.
\end{proof}

\subsection{DAC-HC: an instantiation of DAC}
An instantiation of DAC is presented to illustrate the detail steps of a DAC algorithm and for further empirical studies.

To instantiate DAC, both the decomposition strategy and the sub-problems optimizer should be specified. 
In order to highlight the advantages of the DAC framework further in empirical studies, the improvement of performance introduced by these two specified components should be kept to minimal.
On this basis, we first decompose a problem $f$ via random grouping \cite{yang2008large}. 
That is, in the beginning of each iteration, $M$ equal-sized sub-problems are randomly generated. 
For the sub-problems optimizer, a Parallel Hill Climbing (PHC) method is employed, which thus yields the DAC-Hill Climbing (DAC-HC).
The DAC-HC conducts $N$ RLS processes on each sub-problem.
Specifically, at each iteration of the $i$-th sub-problem, the $j$-th RLS produces one new partial solution $\textbf{x}'_{j;i}$ by applying the Gaussian mutation operator to the current partial solution $\textbf{x}_{j;i}$, using Eq.(\ref{eq8}):

\begin{equation}
\label{eq8}
\textbf{x}'_{j;i} = \textbf{x}_{j;i} + \textbf{I} \cdot \mathcal{N}(0,\sigma_{j;i}).
\end{equation} 

\noindent where $\mathcal{N}(0,\sigma_{j;i})$ denotes a Gaussian random variable with zero mean and standard deviation $\sigma_{j;i}$, and $\textbf{I}$ is the identity matrix of size $d_i$. 
Generally, the value of $\sigma_{j;i}$ represents the search step-size that can be adaptively varied during the search and may also be distinct over RLSs or even dimensions. 
To keep it simple, all RLSs in DAC-HC are initially set to the same search step-size, i.e., 1.00. 
After that, each search step-size is adapted at every iteration in terms of the 1/5 successful rule \cite{kern2004learning}, using Eq.(\ref{eq9}):  

\begin{equation}
\label{eq9}
\sigma_{j;i} = \sigma_{j;i} \times \textrm{exp}^{\frac{1}{\sqrt{D+1}}}(\mathbb{I}_{f(\textbf{x}'_{j;i},\widetilde{\textbf{x}}_{j;\textbf{r}}^{\dagger})\geq f(\textbf{x}_{j;i},\widetilde{\textbf{x}}_{j;\textbf{r}}^{\dagger})}-\frac{1}{5}),
\end{equation} 

\noindent where $\mathbb{I}_{a}$ is a indicator function that returns 1 if $a$ is true and 0 otherwise.

During the selection procedure, each new partial solution in PHC only competes with its corresponding old one for survival. 
Based on this one-on-one relation, we further reduce the FEs consumption of DAC-HC to a half of DAC, i.e., $MN^2$. 
This is conducted by letting two competing partial solutions share the same approximate complement.
The reason behind this is that, by adopting RLSs with small search step-sizes, pairwise partial solutions can be close to each other in the solution space, in which case their respective approximate complements may also be similar.
Lastly, the pseudo-code of DAC-HC is given in Algorithm 2 for illustration.

\begin{algorithm}[t] 
   \caption{DAC-HC($f$, $T_{max}, N, M$)} 
   \label{DAC} 
   \begin{algorithmic}[1] 
       \State Randomly initialize $N$ solutions $\textbf{x}_{1:N}$.  
       \State \textbf{For} $t=1$ \textbf{to} $T_{max}$
           \State \quad Randomly divide $f$ into $M$ equal-sized sub-problems.
           \State \quad \textbf{For} $i = 1$ \textbf{to} $M$ 
           \State \quad\quad \textbf{For} $j = 1$ \textbf{to} $N$
           \State \quad\quad\quad  $\textbf{x}_{j;i}'= \textbf{x}_{j;i} + \textbf{I} \cdot \mathcal{N}(0,\sigma_{j;i})$. 
           \State \quad\quad\quad  $\widetilde{\textbf{x}}_{j;\textbf{r}}^{\dagger} = \underset{\textbf{x}_{\textbf{r}} \in \textbf{x}_{1:N;\textbf{r}}}{\arg\max}f(\textbf{x}_{j;i},\textbf{x}_{\textbf{r}})$.
           \State \quad\quad\quad  $\sigma_{j;i} = \sigma_{j;i} \times \textrm{exp}^{\frac{1}{\sqrt{D+1}}}(\mathbb{I}_{f(\textbf{x}'_{j;i},\widetilde{\textbf{x}}_{j;\textbf{r}}^{\dagger})\geq f(\textbf{x}_{j;i},\widetilde{\textbf{x}}_{j;\textbf{r}}^{\dagger})}-\frac{1}{5})$.
           \State \quad\quad\quad  $\textbf{x}_{j;i} \leftarrow \{\textbf{x}_{j;i}, \textbf{x}'_{j;i}\mid \widetilde{\textbf{x}}_{j;\textbf{r}}^{\dagger}\}$.
           \State \quad\quad\quad  $\textbf{x}_{j;\textbf{r}} \leftarrow \widetilde{\textbf{x}}_{j;\textbf{r}}^{\dagger}$.
           \State \quad\quad \textbf{EndFor}
		   \State \quad \textbf{EndFor}
       \State \textbf{EndFor}  
       \State \textbf{Output} the best solution found so far.
   \end{algorithmic} 
\end{algorithm} 

\section{Empirical Studies}
DAC is proposed for solving non-separable high-dimensional optimization problems.
That is where the empirical studies should concentrate on to verify the effectiveness of DAC-HC.
For this purpose, two sets of non-separable high-dimensional optimization problems are employed. 

\subsection{Varied numbers of interacting variables tests}
The first set of problems is based on the fully non-separable functions, i.e., Schwefel's 1.2 and Rosenbrock \cite{Tang09benchmarkfunctions}, which are formulated as: $f_{\textrm{sch}}(\textbf{x})=\sum^D_{i=1}(\sum^i_{j=1}x_j)^2$ and $f_{\textrm{ros}}(\textbf{x})=\sum^{D-1}_{i=1}[100(x_i^2-x_{i+1})^2+(x_i-1)^2]$. 
In these two functions, all variables are interacting.
Meanwhile, it has also been observed that, in many real-world problems, only parts of variables are interacting \cite{friesen2015recursive}.
Hence, it is a necessity to test DAC-HC with varied numbers of interacting variables.
For this purpose, we further consider three problems that combine Schwefel's 1.2 function and the fully separable sphere function, i.e., $f_{\textrm{sph}}(\textbf{x})=\sum^D_{i=1}x_i^2$, in different formations.
The dimensionality is set to 1000 for all 5 problems. 
All variables are randomly perturbed to avoid any potential bias.
All problems are expected to be minimized to the global optimal value 0.00. 
Specifically, $f_1(\textbf{x})$ consists of a group of 50 interacting variables and 950 independent variables. 
$f_2(\textbf{x})$ has 10 groups of 50 interacting variables and 500 independent variables.
$f_3(\textbf{x})$ and $f_5(\textbf{x})$ compose of 20 groups of 50 interacting variables.
$f_4(\textbf{x})$ involves a group of 1000 interacting variables.
The detailed formulations are listed in Table 1. 

{\renewcommand\baselinestretch{1}\selectfont
\begin{table}[!tbp]\footnotesize
\centering  
\caption{The formulations of 5 tested functions. $\textbf{z}=\textbf{x}-\textbf{o}$ is a shifted solution and $\textbf{o}$ is the global optimum. $\textbf{P}$ is a random permutation of $[1,...,D]$ and $D=1000,m=50$.}
\begin{tabular}{l} 
\hline\hline
$f_1(\textbf{x})=f_{\textrm{sch}}(\textbf{z}(P_1:P_m))*10^6+f_{\textrm{sph}}(\textbf{z}(P_{m+1}:P_D))$ \\ 
$f_2(\textbf{x})=\sum_{k=1}^{\frac{D}{2m}}f_{\textrm{sch}}(\textbf{z}(P_{(k-1)*m+1}:P_{k*m}))$\\
\quad\quad\quad\quad$+f_{\textrm{sph}}(\textbf{z}(P_{\frac{D}{2}+1}:P_D))$ \\ 
$f_3(\textbf{x})=\sum_{k=1}^{\frac{D}{m}}f_{\textrm{sch}}(\textbf{z}(P_{(k-1)*m+1}:P_{k*m}))$ \\ 
$f_4(\textbf{x})=f_{\textrm{sch}}(\textbf{z})$ \\ 
$f_5(\textbf{x})=\sum_{k=1}^{\frac{D}{m}}f_{\textrm{ros}}(\textbf{z}(P_{(k-1)*m+1}:P_{k*m}))$ \\ 
\hline\hline
\end{tabular}
\end{table}
\par}

On these five problems, two groups of comparisons are conducted for different purposes.

\subsubsection{Advantages of DAC-HC over existing approaches}
In the first group of comparison, DAC-HC is compared with  CMA-ES \cite{hansen2001completely}, RESOO \cite{qian2016scaling}, and DECC-I \cite{omidvar2014cooperative}, which are representatives of three basic ideas for high-dimensional optimization: the straightforward method, dimensionality reduction, and DC, respectively. 
Specifically, CMA-ES is widely endorsed as a powerful global optimizer that has been applied in many aspects. 
Here the basic version of CMA-ES is utilized.
RESOO is a recently proposed approach built on the Random Embedding for reducing dimensionality. 
It should be noted that, DECC-I is not an algorithm for black-box optimization. 
It is an ideal approach that perfectly decomposes the problems using the priori knowledge of functions.  
Hence, all sub-problems of DECC-I are independent, while it is not the case for DAC-HC. 
Besides, the sub-problems of DECC-I are optimized by a variant of Differential Evolution \cite{yang2008self}, which has been empirically observed more advanced than the employed RLSs \cite{tang2016negatively}.
On this basis, if DAC-HC outperforms DECC-I, it is reasonable to infer that the proposed DAC facilitates DC on non-separable high-dimensional optimization problems.

All algorithms are repeated 25 runs on each problem to diminish the noise introduced by their randomized search essence.
The time budget for each run is set to 3e6 FEs.
CMA-ES is parameterless that no parameter needs to be specified.
For RESOO, after some coarse-tuning, the probability $\eta$ is set to 1/3, the restart times is set to 10, and the reduced dimensionality is set to 100.
For DECC-I, the only parameter, i.e., population size $N$, is set to 100 as Omidvar \textit{et al.} \shortcite{omidvar2014cooperative} suggested.
For DAC-HC, two parameters should be specified, i.e., the number of RLSs $N$ and the number of sub-problems $M$.
Recall that, the parameter $N>1$ generally influences the approximate ability of DAC.
To test the extreme case of DAC-HC, it is set to 2.
To gain a relatively fair comparison with RESOO, we set $M=10$ so that each sub-problems optimizer faces a 100-dimensional problem as RESOO does. 

The mean and standard derivation of the final outputs in 25 runs are shown in Table 2. 
A gray cell indicates an algorithm achieves the best mean value on a problem, while a light gray cell indicates a second place.
DAC-HC outperforms all the compared algorithms on $f_1,f_4$ and $f_5$.
On $f_2$ and $f_3$, though slightly inferior to CMA-ES, DAC-HC performs significantly better than DECC-I and RESOO.
Since DAC-HC dominates DECC-I on all five problems, the effectiveness of DAC for promoting DC on non-separable high-dimensional optimization problems is confirmed.
Besides, Lemma \ref{le1} is verified by observing that the solution qualities of DAC-HC deteriorates as the number of interacting variables increases.
RESOO performs poorly because the tested problems are irreducible.

{\renewcommand\baselinestretch{1}\selectfont
\begin{table}[tbp]\footnotesize
\centering  
\caption{The mean and standard derivation of final outputs.}

\begin{tabular}{cccccc}
\hline\hline
\multicolumn{2}{c}{\textbf{Function}} &\textbf{CMA-ES} &\textbf{RESOO} &\textbf{DECC-I} &\textbf{DAC-HC}\\ \hline
\multirow{2}*{$f_1$} 
&Mean   &1.35e+09  &2.52e+11  &\cellcolor{lightgray} 2.97e+01  &\cellcolor{gray} 0.00e+00  \\ 
&Std    &3.29e+08  &3.62e+10  &8.59e+01  &0.00e+00  \\ 

\multirow{2}*{$f_2$} 
&Mean   &\cellcolor{gray} 0.00e+00  &1.28e+07  &1.48e+03  &\cellcolor{lightgray} 1.55e+00  \\ 
&Std    &0.00e+00  &9.41e+05  &4.28e+02  &1.25e+00  \\ 

\multirow{2}*{$f_3$} 
&Mean   &\cellcolor{gray} 0.00e+00  &3.62e+07  &3.91e+04  &\cellcolor{lightgray} 7.78e+02  \\ 
&Std    &0.00e+00  &2.89e+06  &2.75e+03  &7.12e+02  \\ 

\multirow{2}*{$f_4$} 
&Mean   &2.87e+06  &7.80e+07  &\cellcolor{lightgray} 1.74e+06  &\cellcolor{gray} 3.93e+05  \\ 
&Std    &6.61e+05  &7.10e+06  &9.54e+04  &2.52e+04  \\ 

\multirow{2}*{$f_5$} 
&Mean   &3.36e+03  &1.41e+12  &\cellcolor{lightgray} 1.17e+03  &\cellcolor{gray} 1.13e+03  \\ 
&Std    &1.81e+03  &8.02e+09  &9.66e+01  &2.32e+02  \\ \hline
 \hline
\end{tabular}
\end{table}
\par}

\subsubsection{Empirical support to the convergence of DAC}
In the second group of comparison, the PHC with 2 RLSs is compared.
The only difference between PHC and DAC-HC is that PHC complements a partial solution $\textbf{x}_{j;i}$ merely with the corresponding partial solutions in the rest sub-problems $\textbf{x}_{j;\textbf{r}}$.
On this basis, PHC does not satisfy Eq.(\ref{eq6}), and its convergence is not guaranteed.
The experimental protocol is set the same to the first group of comparison.

The convergence rates of both algorithms are shown in Figure 2, where the x-axis denotes the FEs and y-axis denotes the logarithm of function values.
It can be seen that, DAC-HC always converges faster than PHC.
Specially, log-linear convergence of DAC-HC is observed on the first two sphere function based problems.
It should be noted that, the employed RLS has also been theoretically proved to converge log-linearly on the sphere function \cite{jebalia2008log}.
This coincidence actually supports the convergence of DAC stated in Lemma \ref{le2}.   
Comparatively, the convergence rates of PHC are heavily retarded due to the unfit complements to partial solutions.

\begin{figure}\renewcommand{\captionfont}{\footnotesize}
			 \centering
			\begin{minipage}[htb]{1\linewidth}
				\centering
				\includegraphics[width=3in,height=3in]{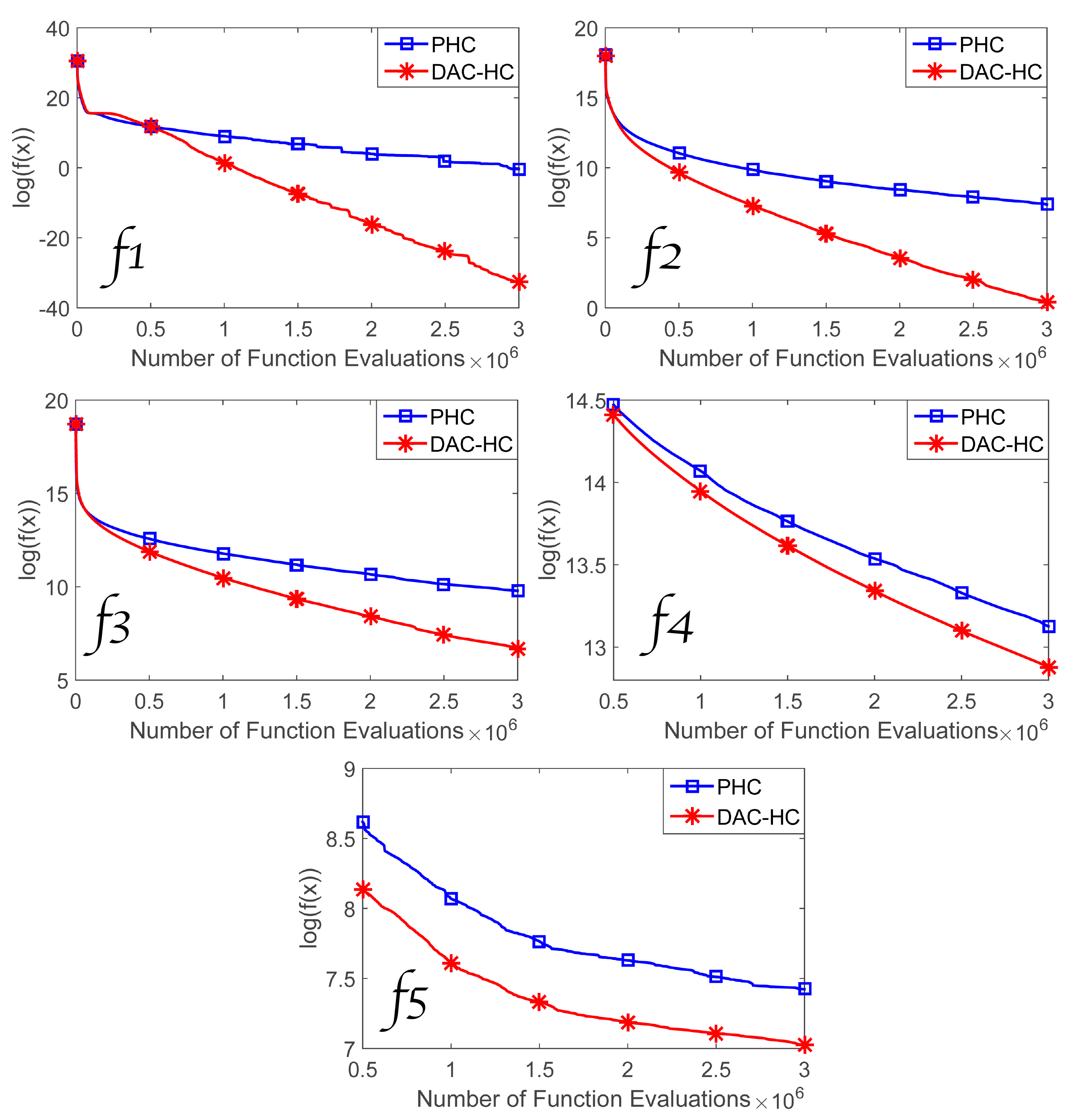}
				\caption{The convergence rates of DAC-HC and PHC. The x-axis denotes the FEs (time) and y-axis denotes $\log(f(\textbf{x}))$.}
			\end{minipage}%
\end{figure}

\subsection{Hyper-parameter tuning for multi-class SVMs}
Given a set of labelled data $\{\textbf{x}_i,y_i\}_{i=1}^l$, where $\textbf{x}_i \in R^n$, the classification task is to train a classifier in terms of $\{\textbf{x}_i,y_i\}_{i=1}^l$ to predict the labels of incoming data.
Support Vector Machines (SVMs) \cite{vapnik1998statistical} is often considered as a family of powerful tools for classification.
Here the SVM with linear kernel is considered.
Let $\textbf{y}\in\{-1,+1\}$, SVM requires to fine-tune three parameters $\textbf{w},b,\lambda$ by solving the following optimization problem:
${\min}_{\textbf{w},b,\lambda} \frac{1}{2}\textbf{w}^T\textbf{w}+\lambda\sum^l_{i=1}\xi_i$, subject to 
$y_i(\textbf{w}^T\textbf{x}_i+b) \geq 1-\xi_i \land \xi_i \geq 0$.
Notice that, $\lambda>0$ is a hyper-parameter supplied by the user, which penalizes the error vector $\boldsymbol{\xi}$.
When dealing with a multi-class classification problem, a typical idea is to divide it into multiple binary classification problems by adopting the one-on-one strategy.
Let $K$ be the number of class, then we have ${K \choose 2}=\frac{K(K-1)}{2}$ binary classifiers to train, resulting in $\frac{K(K-1)}{2}$ hyper-parameters to tune.

Of course, a simple strategy of specifying the same value of $\lambda$ for all binary classifiers works well \cite{CC01a}.
It is also an intuition that varied $\boldsymbol{\lambda}$ can facilitate a multi-class SVM better.
On this basis, a potentially high-dimensional optimization problem needs to be solved for more advanced performance.
Recall that, a final output of a multi-class SVM is based on the majority voting.
Due to the overfitting risk on each binary classifier, the votes may introduce interdependencies in between.
To sum up, the problem of hyper-parameter tuning for multi-class SVMs is non-separable and high-dimensional in essence.

We thus apply the proposed DAC-HC to deal with it. 
RESOO and CMA-ES are again included as the compared algorithms. 
Since the ideal decomposition is no longer applicable for this problem, DECC-I will not be compared with DAC-HC.
The grid search is also tested but on the assumption that all binary classifiers share the same value of $\lambda$.
That is , the grid search actually solves a 1-dimensional problem.
Three data sets, i.e., \textit{usps} \cite{hull1994database}, \textit{news20} \cite{lang1995newsweeder}, and \textit{letter} \cite{hsu2002comparison}, are used for comparison, which contain 10, 20, 26 classes in each and yield three problems with 45, 190, 325 dimensions, respectively. 
All the features in each dataset are scaled to $[-1,1]$ or $[0,1]$.
The solution space for each binary classifier is bounded as $[10^{-3},10^2]$. 
All algorithms are repeated for 20 runs on each problem, except the deterministic grid search. 
For each run, the hyper-parameters are tuned on the training set with the 5-fold cross-validation.
The higher the accuracy is, the better the hyper-parameters are supposed to be.
The best hyper-parameters obtained in a run will be tested on the testing set, and the testing accuracy is regarded as the performance of an algorithm in a run. 
The time budget for the training phase in each run is set to 100 FEs.
For RESOO, the probability $\eta$ is set to 1/3, the restart times is set to 2, and the reduced dimensionality is set to 1/5 to the original dimensionality, i.e., 9, 38, 65, respectively.
For the grid search, 100 candidate solutions are uniformly selected over the solution space.
For DAC-HC, the parameters $N$ and $M$ are set to 2 and 3, respectively.

{\renewcommand\baselinestretch{1}\selectfont
\begin{table}[tbp]\footnotesize
\centering  
\caption{Testing accuracies of tuned multi-class SVM.}

\begin{tabular}{cccccc}
\hline\hline
\textbf{DataSet} &\textbf{GridSearch} &\textbf{RESOO} &\textbf{CMA-ES} &\textbf{DAC-HC}\\ \hline

\multirow{2}*{\textit{usps}} 
&93.92$\%$  &\cellcolor{lightgray} 94.38$\%$  &93.33$\%$  &\cellcolor{gray} 94.60$\%$  \\ 
&-    &$\pm$0.37$\%$  &$\pm$0.24$\%$   &$\pm$0.04$\%$  \\ 

\multirow{2}*{\textit{news20}} 
&85.16$\%$  &\cellcolor{gray} 85.43$\%$  &84.34$\%$  &\cellcolor{lightgray} 85.40$\%$  \\ 
&-    &$\pm$0.21$\%$  &$\pm$0.17$\%$   &$\pm$0.11$\%$  \\ 

\multirow{2}*{\textit{letter}} 
&85.12$\%$ &\cellcolor{lightgray} 85.36$\%$  &84.03$\%$  &\cellcolor{gray} 85.72$\%$  \\ 
&-    &$\pm$0.12$\%$  &$\pm$1.44$\%$   &$\pm$0.24$\%$  \\  

 \hline
\end{tabular}
\end{table}
\par}

Table 3 lists the mean and standard derivation of the testing accuracies of the multi-class SVMs tuned by each algorithm.
It can be seen that, both RESOO and DAC-HC outperform the grid search, although the improvement is marginal.
Hence, it can be inferred that tuning multiple hyper-parameters, rather than one shared hyper-parameter, is beneficial to the multi-class SVMs.
DAC-HC outperforms all the compared algorithms on the $usps$ and $letter$ datasets.
Although it shows slightly lower accuracy than RESOO on the $news20$ dataset, a more stable behavior is observed as its standard derivation is smaller.
CMA-ES is inferior to the grid search on all three problems.
This phenomenon suggests that, for tuning multiple hyper-parameters for multi-class SVMs, an appropriate optimization approach should be employed.
Otherwise, it would be better to tune just one shared hyper-parameter.
It is also worthwhile to notice that, CMA-ES does not adopt any special treatment for high-dimensional optimization.

\section{Conclusions and Future Directions}
This work investigated the Divide and Conquer idea on high-dimensional black-box optimization problems.
We found that the interdependent sub-problems after decomposition actually cannot be accurately conquered.
Instead, we proposed the Divide and Approximate Conquer (DAC) to solve each sub-problem approximately.
The convergence of DAC was proved and empirically supported.
For empirical studies, a simple instantiation of DAC, i.e., DAC-HC, was also proposed.
The advantages of DAC-HC over existing representative approaches were verified on two sets of non-separable high-dimensional problems.

For future work, we are interested in:

\begin{itemize}
\item Promoting the ability of DAC by adopting more advanced sub-problems optimizers.
\item Theoretically analyzing the convergence rate of DAC.  
\end{itemize}


\appendix

\bibliographystyle{named}
\bibliography{ijcai16}

\end{document}